\documentclass[11pt]{article}
\usepackage[margin=1in]{geometry}

\usepackage[backref,colorlinks,citecolor=blue,bookmarks=true]{hyperref}
\usepackage{mathtools, amssymb, amsthm, bbm}
\usepackage[ruled,vlined]{algorithm2e}
\usepackage{svg}
 \usepackage{tikz}
\SetAlCapSkip{1em}

\title{Agnostically Learning Single-Index Models using Omnipredictors}

\author{%
  Aravind Gollakota\thanks{\texttt{aravindg@cs.utexas.edu}. Supported by NSF award AF-1909204 and the NSF AI Institute for Foundations of Machine Learning (IFML).} \\
  UT Austin \\
  \and
  Parikshit Gopalan\thanks{\texttt{parikg@apple.com}.} \\
  Apple \\
  \and
  Adam R. Klivans\thanks{\texttt{klivans@cs.utexas.edu}. Supported by NSF award AF-1909204 and the NSF AI Institute for Foundations of Machine Learning (IFML).} \\
  UT Austin \\
  \and
  Konstantinos Stavropoulos\thanks{\texttt{kstavrop@cs.utexas.edu}. Supported by NSF award AF-1909204, the NSF AI Institute for Foundations of Machine Learning (IFML), and by scholarships from Bodossaki Foundation and Leventis Foundation.} \\
  UT Austin
}
\usepackage{amsthm, amsmath, amssymb, bbm, graphicx, url}
\theoremstyle{plain}
\newtheorem{theorem}{Theorem}[section]
\newtheorem{lemma}[theorem]{Lemma}

\newtheorem{proposition}[theorem]{Proposition}

\theoremstyle{definition}
\newtheorem{definition}[theorem]{Definition}

\theoremstyle{remark}

\numberwithin{equation}{section}

\newcommand{\wh}{\widehat}
\newcommand{\R}{\mathbb{R}}

\newcommand{\mH}{\mathcal{H}}

\newcommand{\x}{\mathbf{x}}

\newcommand{\w}{\mathbf{w}}

\newcommand{\logf}[2]{\ensuremath{\log\left(\frac{#1}{#2}\right)}}

\newcommand{\zo}{\ensuremath{\{0,1\}}}
\newcommand{\izo}{\ensuremath{[0,1]}}

\renewcommand{\hat}{\wh}

\newcommand{\opt}{\mathrm{opt}}
\newcommand{\C}{\mathcal{C}}

\DeclareMathOperator{\poly}{poly}

\def\Ball{\mathbb{B}}
\def\Djoint{D}%\mathcal{D}}
\def\range{\textnormal{ran}}

\def\mY{\mathcal{Y}}

\def\ind{\mathbbm{1}}
\DeclareMathOperator*{\ex}{\mathbb{E}}

\def\elltwoerror{\mathrm{err}_2}
\def\elloneerror{\mathrm{err}_1}
\def\ws{\mathcal{W}}

\def\link{u}
\def\Dmarginal{D_{\x}}%\mathcal{D}_{\X}}
\def\dparam{\lambda}
\def\vv{\mathbf{v}}
\def\S{\mathbb{S}}
\def\N{\mathbb{N}}
\def\f{f}
\def\g{g}
\DeclareMathOperator*{\pr}{\mathbb{P}}
\def\L{\mathcal{L}}
\def\wnorm{B}
\def\vv{\mathbf{v}}
\def\speedparam{\gamma}
\def\fenchelpairs{\mathcal{F}}
\def\fbound{R}
\DeclareMathOperator{\Bregman}{\mathbb{D}}
\def\expon{\gamma}
\def\kl{\Bregman_{\textnormal{KL}}}
\def\crossentropy{\mathrm{CE}}

\DeclareMathOperator{\simclass}{SIM}
\DeclareMathOperator{\glmclass}{GLM}
\newcommand{\signal}{\mathrm{signal}}
\newcommand{\random}{\mathrm{random}}

\begin{document}

\maketitle

\begin{abstract}
    We give the first result for agnostically learning Single-Index Models (SIMs) with arbitrary monotone and Lipschitz activations. All prior work either held only in the realizable setting or required the activation to be known. Moreover, we only require the marginal to have bounded second moments, whereas all prior work required stronger distributional assumptions (such as anticoncentration or boundedness). Our algorithm is based on recent work by \cite{GopalanHKRW23} on omniprediction using predictors satisfying calibrated multiaccuracy. Our analysis is simple and relies on the relationship between Bregman divergences (or matching losses) and $\ell_p$ distances. We also provide new guarantees for standard algorithms like GLMtron and logistic regression in the agnostic setting. 
\end{abstract}
\newpage

\section{Introduction}

Generalized Linear Models (GLMs) and Single-Index Models (SIMs) constitute fundamental frameworks in supervised learning \cite{mccullagh1984glms, agresti2015foundations}, capturing and generalizing basic models such as linear and logistic regression. In the GLM framework, labeled examples $(\x,y)$ are assumed to satisfy $\ex[y|\x] = {\link}^{-1}(\w\cdot \x)$, where ${\link}$ is a known monotone function (called the link function) and $\w$ is an unknown vector. Single-Index Models (SIMs) are defined similarly, but for the case when the monotone link function ${\link}$ is also unknown.

In the realizable setting where the labels are indeed generated according to a GLM with a given Lipschitz link function, the GLMtron algorithm of \cite{kakade2011efficient} is a simple and efficient learning algorithm. When the ground truth is only assumed to be a SIM (and, hence, the link function is unknown), it can be learned efficiently by the Isotron algorithm \cite{KalaiS09,kakade2011efficient}.

In this work, we consider the significantly more challenging agnostic setting, where the labels are arbitrary and not necessarily realizable. Moreover, we do not assume that we know the optimal activation; our goal is to output a predictor that has error comparable to that of the optimal SIM, whatever its activation may be. That is, we must be competitive against not only all possible weights but also all possible monotone and Lipschitz link functions that might fit the distribution. Concretely, consider a distribution $\Djoint$ over $\R^d\times [0,1]$ and denote the squared error of a function $h:\R^d\to \R$ by $\elltwoerror(h) = \ex_{(\x,y)\sim\Djoint}[(y - h(\x))^2]$. Let $\opt(\simclass)$ denote optimal value of $\elltwoerror(h)$ over all SIMs $h$ with bounded weights and arbitrary $1$-Lipschitz monotone activations (we call the inverse $u^{-1}$ of a link function $u$ the activation function). Given access to samples from $\Djoint$, our goal is to compute a predictor $p:\R^d\to [0,1]$ with error $\elltwoerror(p)$ comparable to the error of the optimal SIM:
\[
    \elltwoerror(p) \leq \opt(\simclass) + \epsilon.
\]

Our main result is the first efficient learning algorithm for this problem.
\begin{theorem}[Informal, see Theorem \ref{theorem:main}]\label{informal-theorem:main-result}
    Let $\simclass_{B}$ denote the class of SIMs of the form $\x \mapsto \link^{-1}(\w \cdot \x)$ for some $1$-Lipschitz function $\link^{-1}$ and $\|\w\|_2\le B$. Let $\Djoint$ be any distribution over $\R^d\times [0,1]$ whose marginal on $\R^d$ has bounded second moments. There is an efficient algorithm that agnostically learns $\simclass_B$ over $\Djoint$ up to error
    \[
        \elltwoerror(p) \le O\big(B\sqrt{\opt(\simclass_B)}\big) + \epsilon.
    \]
\end{theorem}

This result provides a guarantee comparable to that of the Isotron algorithm \cite{KalaiS09,kakade2011efficient} but for the challenging agnostic setting rather than the realizable setting (where $\opt(\simclass_B,\Djoint)= 0$). Moreover, Isotron's guarantees require the distribution to be supported on the unit ball, whereas we only require a mild second moment condition.

Minimizing the squared error in the agnostic setting is a standard benchmark in learning theory, but it is often useful to simplify the problem by considering an alternative error function, tailored to the specific problem at hand. In this sense, we can still find the GLM that is closest to the ground truth even in the agnostic setting; the key is to define closeness using an appropriate Bregman divergence, which depends on the link function through Fenchel-Legendre convex duality. 

The problem of finding the closest model in Bregman divergence amounts to a convex program where we minimize a certain loss called the matching loss \cite{Auer95exponentiallymany}. In fact, recent work by \cite{gopalan2022omnipredictors,GopalanHKRW23} 
studying the notion of omniprediction has demonstrated that there exist efficient algorithms that minimize all of the matching losses (corresponding to monotone and bounded links) simultaneously. Their solution concept is called an omnipredictor, i.e., a single predictor that is able to compete with the best-fitting classifier in a class $\C$ as measured by a large class of losses (as opposed to just a single pre-specified loss, as is standard in machine learning). They obtain such predictors through calibrated multiaccuracy \cite{GopalanHKRW23} or multicalibration \cite{gopalan2022omnipredictors}.
Their results apply to the non-realizable setting and do not assume prior knowledge of the link function, but only provide guarantees for (simultaneous) matching loss minimization, rather than the standard squared error minimization.

We propose a simple analytic approach to transforming matching loss guarantees over the class of linear functions to squared error guarantees over the class of GLMs with link function that corresponds to the matching loss at hand. Our generic transformation, coupled with the omniprediction results from \cite{GopalanHKRW23}, yields our main result on agnostically learning SIMs. We thus obtain a best of all worlds statement: we do not need to know the link function, but we can always compete with the best SIM model in terms of the squared loss.
At the heart of our approach are distortion inequalities relating matching losses to $\ell_p$ losses that we believe may be of independent interest. 

We first prove strong versions of such inequalities for matching losses arising from bi-Lipschitz link functions, and obtain our results for general Lipschitz activations by carefully approximating them using bi-Lipschitz activations. In particular, if we let $\opt(\glmclass_{u^{-1},B})$ denote the optimal value of $\elltwoerror(h)$ over all GLMs of the form $\x \mapsto u^{-1}(\w\cdot \x)$, where $\|\w\|_2\le B$, we obtain the following result about bi-Lipschitz activations (including, for example, the Leaky ReLU activation).

\begin{theorem}[Informal, see Theorem \ref{corollary:linear-surrogate-squared-expected}]
    Let $\link:\R\to\R$ be a bi-Lipschitz invertible link function. Then, any predictor $p:\R^d\to\R$ that is an $\epsilon$-approximate minimizer of the population matching loss that corresponds to $u$, with respect to a distribution $\Djoint$ over $\R^d\times[0,1]$ satisfies
    \[
        \elltwoerror(p) \le O\big(\opt(\glmclass_{u^{-1},B})\big) + O(\epsilon)
    \]
\end{theorem}

This guarantee holds under milder distributional assumptions than are required by comparable prior work on agnostically learning GLMs \cite{frei2020agnostic,diakonikolas22learningsingleneuron}. Moreover, when we focus on distortion bounds between the logistic loss and the squared loss, we obtain a near-optimal guarantee of $\widetilde{O}(\opt(\glmclass_{u^{-1},B}))$ for logistic regression, when $u$ is the logit link function (i.e., when $\glmclass_{u^{-1},B}$ is the class of sigmoid neurons). 

\begin{theorem}[Informal, see Theorem \ref{theorem:logistic-squared-expected}]
    Let $\link(t) = \ln(\frac{t}{1-t})$. Then, any predictor $p:\R^d\to \R$ that is an $\epsilon$-approximate minimizer of the population logistic loss, with respect to a distribution $\Djoint$ over $\R^d\times[0,1]$ whose marginal on $\R^d$ has subgaussian tails in every direction satisfies
    \[
        \elltwoerror(p) \le \widetilde{O}\big(\opt(\glmclass_{u^{-1},B})\big) + O(\epsilon)
    \]
\end{theorem}
While our error guarantee is weaker the one of \cite{diakonikolas22learningsingleneuron}, we do not make the anti-concentration assumptions their results require.
A natural question is to ask if our guarantees are near-optimal, e.g., whether we can obtain a guarantee of the form $\elltwoerror(p)\le \opt(\simclass) + \epsilon$. However, there is strong evidence that such results cannot be obtained using efficient algorithms \cite{goel2019time,diakonikolas2020near,goel2020statistical,diakonikolas21theoptimality}. Adapting a result due to \cite{diakonikolas22hardness}, we show in Section~\ref{sec:lower-bound} that one cannot avoid a dependence on the norm bound $B$ in our main result, Theorem~\ref{informal-theorem:main-result}.

\subsection{Background and Relation to Prior Work}

We note that matching losses have been studied in various previous works either implicitly \cite{kakade2011efficient} or explicitly \cite{Auer95exponentiallymany,diakonikolas2020approximation,GopalanHKRW23} and capture various fundamental algorithms like logistic and linear regression \cite{mccullagh1984glms,agresti2015foundations}. However, to the best of our knowledge, our generic and direct approach for transforming matching loss guarantees to squared error bounds, has not been explored previously. Furthermore, our results do not depend on the specific implementation of an algorithm, but only on the matching loss bounds achieved by its output. In this sense, we provide new agnostic error guarantees for various existing algorithms of the literature. 
For example, our results imply new guarantees for the GLMtron algorithm of \cite{kakade2011efficient} in the agnostic setting, since GLMtron can be viewed as performing gradient descent (with unit step size) on the matching loss corresponding to a specified link function.

Matching losses over linear functions are also linked to the Chow parameters \cite{o2008chow} through their gradient with respect to $\w$, as observed by \cite{diakonikolas2020approximation}. In fact, the norm of the matching loss gradient is also linked to multiaccuracy, a notion that originates to fairness literature \cite{hebert18multicalibration,kim2019multiaccuracy}. A stationary point $\w$ of a matching loss that corresponds to a GLM with link $u$ turns out to be a multiaccurate predictor $p(\x) = u^{-1}(\w\cdot \x)$, i.e., a predictor such that for all $i \in [d]$, $\ex[\x_i (y-p(\x))] = 0$. 
The work of \cite{gopalan2022omnipredictors,GopalanHKRW23} on omnipredictors presents a single predictor that is better than any linear model $\w \cdot \x$ for every matching loss. The results of \cite{gopalan2022omnipredictors} show that a multicalibrated predictor (with respect to the features $\x_i$) is an omnipredictor for all convex losses, whereas \cite{GopalanHKRW23} shows that a simpler condition of calibrated multiaccuracy suffices for matching losses that arise from GLMs. 
In view of the relationship between multiaccuracy and the gradient of the matching loss, our results show that, while multiaccuracy implies bounds on agnostically learning GLMs, the additional guarantee of calibration is sufficient for agnostically learning all SIMs. 

The work of \cite{shalev2011learning} showed strong agnostic learning guarantees in terms of the absolute error (rather than the squared error) of the form $\opt + \epsilon$ for a range of GLMs, but their work incurs an exponential dependence on the weight norm $B$. For the absolute loss, we obtain a bound of the form $B\,\opt \log(1/\opt)$ for logistic regression (see Theorem \ref{theorem:logistic-absolute-expected}). In more recent years, the problem of agnostically learning GLMs has frequently also been phrased as the problem of agnostically learning single neurons (with a known activation). For the ReLU activation, work by \cite{goel2017reliably} showed an algorithm achieving error $\opt + \epsilon$ in time $\poly(d)\exp(1/\epsilon)$ over marginals on the unit sphere, and \cite{diakonikolas2020approximation} showed an algorithm achieving error $O(\opt) + \epsilon$ in fully polynomial time over isotropic log-concave marginals. The work of \cite{frei2020agnostic,diakonikolas22learningsingleneuron} both show guarantees for learning general neurons (with known activations) using the natural approach of running SGD directly on the squared loss (or a regularized variant thereof). \cite{frei2020agnostic} achieves error $O(\opt)$ for strictly increasing activations and $O(\sqrt{\opt})$ for the ReLU activation over bounded marginals, while \cite{diakonikolas22learningsingleneuron} proved an $O(\opt)$ guarantee for a wide range of activations (including the ReLU) and over a large class of structured marginals. 

In terms of lower bounds and hardness results for this problem, the work of \cite{goel2019time,diakonikolas2020near,goel2020statistical,diakonikolas21theoptimality,diakonikolas22hardness} has established superpolynomial hardness even for the setting of agnostically learning single ReLUs over Gaussian marginals.

\paragraph{Limitations and directions for future work.} While we justify a certain dependence on the norm bound $B$ in our our main result on agnostically learning SIMs, we do not provide tight lower bounds corresponding to Theorem~\ref{informal-theorem:main-result}. An important direction for future work is to tightly characterize the optimal bounds achievable for this problem, as well as to show matching algorithms.

\subsection{Preliminaries}
For the following, $(\x,y)$ is used to denote a labelled sample from a distribution $\Djoint$ over $\R^d\times \mY$, where $\mY$ denotes the interval $[0,1]$ unless it is specified to be the set $\{0,1\}$. We note that, although we provide results for the setting where the labels lie within $[0,1]$, we may obtain similar results for any bounded label space. We use $\pr_\Djoint$ (resp. $\ex_\Djoint$) to denote the probability (resp. expectation) over $\Djoint$ and, similarly, $\pr_S$ (resp. $\ex_S$) to denote the corresponding empirical quantity over a set $S$ of labelled examples. Throughout the paper, we will use the term differentiable function to mean a function that is differentiable except on finitely many points.

Our main results will assume the following about the marginal distribution on $\R^d$.
\begin{definition}[Bounded moments]
    For $\dparam>0$ and $k\in\N$, we say that a distribution $\Dmarginal$ over $\R^d$ has $\dparam$-bounded $2k$-th moments if for any $\vv\in \S^{d-1}$ we have $\ex_{\x\sim \Dmarginal}[(\vv\cdot \x)^{2k}] \le \dparam$.
\end{definition}
For a concept class $\C:\R^d\to \R$, we define $\opt(\C,\Djoint)$ to be the minimum squared error achievable by a concept $c:\R^d\to \R$ in $\C$ with respect to the distribution $\Djoint$.

We will also provide results that are specific to the sigmoid activation and work under the assumption that the marginal distribution is sufficiently concentrated.
\begin{definition}[Concentrated marginals]\label{definition:concentration}
    For $\dparam>0$ and $\expon$, we say that a distribution $\Dmarginal$ over $\R^d$ is $(\dparam,\expon)$-concentrated if for any $\vv\in \S^{d-1}$ and $r\ge 0$ we have $\pr_{\x\sim \Dmarginal}[|\vv\cdot \x|\ge r] \le \dparam \cdot \exp({-r^\expon})$.
\end{definition}

\begin{definition}[Fenchel-Legendre pairs]
We call a pair of functions $(\f,\g)$ a Fenchel-Legendre pair if the following conditions hold.
\begin{enumerate}
    \item\label{property:pairfg1} $\g':\R\to\R$ is continuous, non-decreasing, differentiable and $1$-Lipschitz with range $\range(\g')\supseteq(0,1)$ and $\g(t) = \int_{0}^t \g'(\tau) \; d\tau$, for any $t\in\R$.
    \item\label{property:pairfg2} $\f:\range(\g')\to \R$ is the convex conjugate (Fenchel-Legendre transform) of $\g$, i.e., we have $\f(r) = \sup_{t\in\R} r\cdot t - \g(t)$ for any $r\in\range(\g')$.
\end{enumerate}
\end{definition}
For such pairs of functions, the following are true for $r\in \range(\g')$ and $t\in \range(\f')$ (note that $\range(\f')$ is not necessarily $\R$ when $\g'$ is not invertible).
\begin{align}
    \g'(\f'(r)) &= r \,\text{ and }\; \f(r) = r\f'(r)-\g(\f'(r)), \text{ for } r\in \range(\g') \label{equation:fenchel-inverse-gof}\\
    \f'(\g'(t)) &= t \;\text{ and }\; \g(t) = t\g'(t)-\f(\g'(t)), \text{ for } t\in \range(\f')\label{equation:fenchel-inverse-fog-plus}
\end{align}
Note that $\g'$ will be used as an activation function for single neurons and $\f'$ corresponds to the unknown link function of a SIM (or the known link function of a GLM). We say that $\g'$ is bi-Lipschitz if for any $t_1 < t_2 \in\R$ we have that ${(\g'(t_2)-\g'(t_1))}/{(t_2-t_1)}\in [\alpha , \beta ]$. If $\g'$ is $[\alpha,\beta]$ bi-Lipschitz, then $\f'$ is $[\frac{1}{\beta},\frac{1}{\alpha}]$ bi-Lipschitz. However, the converse implication is not necessarily true when $\g'$ is not strictly increasing.

\begin{definition}[Matching Losses]
For a non-decreasing and Lipschitz activation $\g':\R\to\R$, the matching loss $\ell_{\g}:\mY\times\R\to\R$ is defined pointwise as follows:
\[
    \ell_{\g}(y,t) = \int_{0}^{t} \g'(\tau) - y\; d\tau,
\] where $g(t) = \int_0^{t} \g'$. The function $\ell_{\g}$ is convex and smooth with respect to its second argument. 
The corresponding population matching loss is
\begin{equation}\label{equation:surrogate-loss}
    \L_{\g}(c\;;\Djoint) = \ex_{(\x,y)\sim\Djoint}\Bigr[\ell_{\g}(y,c(\x))\Bigr]
\end{equation}
In Equation \eqref{equation:surrogate-loss}, $c:\R\to \R$ is some concept and $\Djoint$ is some distribution over $\R^d\times[0,1]$.
In the specific case where $c$ is a linear function, i.e., $c(\x) = \w\cdot \x$, for some $\w\in\R^d$, then we may alternatively denote $\L_{\g}(c\;;\Djoint)$ with $\L_{\g}(\w\,;\Djoint)$.
\end{definition}
We also define the Bregman divergence associated with $\f$ to be $\Bregman_{\f}(q, r) = f(q) - f(r) - (q - r){f'}(r), $ for any $q,r\in\range(g')$. Note that $\Bregman_{\f}(q, r) \ge 0$ with equality iff $q=r$.

\begin{definition}[SIMs and GLMs as Concept Classes]
    For $\wnorm>0$, we use $\simclass_B$ to refer to the class of all SIMs of the form $\x \mapsto g'(\w \cdot \x)$ where $\|\w\|_2 \leq B$ and $g' : \R \to \R$ is an arbitrary $1$-Lipschitz monotone activation that is differentiable (except possibly at finitely many points). We define $\glmclass_{g', B}$ similarly except for the case where $g'$ is fixed and known.
\end{definition}

We also define the notion of calibrated multiaccuracy that we need to obtain omnipredictors in our context.
\begin{definition}[Calibrated Multiaccuracy]
    A predictor $p : \R \to [0, 1]$ is called $\epsilon$-multiaccurate if for all $i \in [d]$, $|\ex[\x_i (y-p(\x))]| \leq \epsilon$. It is called $\epsilon$-calibrated if $|\ex_{p(\x)} \ex_{y | p(\x)}[y - p(\x)]| \leq \epsilon$.
\end{definition}

\section{Distortion Bounds for the Matching Loss}\label{section:distortion-bounds}

In this section, we propose a simple approach for bounding the squared error of a predictor that minimizes a (convex) matching loss, in the agnostic setting. We convert matching loss bounds to squared loss bounds in a generic way, through appropriate pointwise distortion bounds between the two losses. In particular, for a given matching loss $\L_\g$, we transform guarantees on $\L_\g$ that are competitive with the optimum linear minimizer of $\L_\g$ to guarantees on the squared error that are competitive with the optimum GLM whose activation ($\g'$) depends on the matching loss at hand.

We now provide the main result we establish in this section.

\begin{theorem}[Squared Error Minimization through Matching Loss Minimization]\label{corollary:linear-surrogate-squared-expected}
    Let $\Djoint$ be a distribution over $\R^d\times [0,1]$, let $0<\alpha\le \beta$ and let $(\f,\g)$ be a Fenchel-Legendre pair such that $g':\R\to\R$ is $[\alpha, \beta]$ bi-Lipschitz. Suppose that for a predictor $p:\R^d \to \range(\g')$ we have
    \begin{equation}\label{equation:linear-assumption-theorem-surrogate-squared-expected}
        \L_{\g}(\f'\circ p\,;\Djoint) \le \min_{\|\w\|_2 \leq \wnorm}\L_{\g}(\w\,;\Djoint) + \epsilon
    \end{equation}
    Then we also have: $\elltwoerror(p) \le \frac{\beta}{\alpha} \cdot \opt(\glmclass_{g', \wnorm}) + 2\beta \epsilon$.
\end{theorem}

The proof of Theorem \ref{corollary:linear-surrogate-squared-expected} is based on the following pointwise distortion bound between matching losses corresponding to bi-Lipschitz link functions and the squared distance.

\begin{lemma}[Pointwise Distortion Bound for bi-Lipschitz link functions]\label{lemma:surrogate-squared-pointwise}
    Let $0<\alpha\le \beta$ and let $(\f,\g)$ be a Fenchel-Legendre pair such that $\f':\range(\g')\to \R$ is $[\frac{1}{\beta},\frac{1}{\alpha}]$ bi-Lipschitz. Then for any $y, p\in\range(\g')$ we have
    \[
        \ell_{\g}(y,{\f'}(p)) - \ell_{\g}(y,{\f'}(y)) = \Bregman_{\f}(y,p) \in \left[ \frac{1}{2\beta}(y-p)^2 , \frac{1}{2\alpha} (y-p)^2 \right]
    \]
\end{lemma}

In the case that $\f'$ is differentiable on $(0,1)$, the proof of Lemma \ref{lemma:surrogate-squared-pointwise} follows from an application of Taylor's approximation theorem of degree $2$ on the function $\f$, since the Bregman divergence $\Bregman_{\f}(y,p)$ is exactly equal to the error of the second degree Taylor's approximation of $\f(y)$ around $p$ and $\f''(\xi)\in[\frac{1}{\beta},\frac{1}{\alpha}]$ for any $\xi\in\range(\g')$. The relationship between $\ell_\g$ and $\Bregman_\f$ follows from property \eqref{equation:fenchel-inverse-fog-plus}. Note that when $\g'$ is $[\alpha,\beta]$ bi-Lipschitz, then $\f'$ is $[\frac{1}{\beta},\frac{1}{\alpha}]$ bi-Lipschitz.

Theorem \ref{corollary:linear-surrogate-squared-expected} follows by applying Lemma \ref{lemma:surrogate-squared-pointwise} appropriately to bound the error of a predictor $p$ by its matching loss $\L_\g(\f'\circ p)$ and bound the matching loss of the linear function corresponding to $\w^*$ by the squared error of $\g'(\w^*\cdot \x)$, where $\g'(\w^*\cdot \x)$ is the element of $\glmclass_{\g',\wnorm}$ with minimum squared error.

Although Theorem \ref{corollary:linear-surrogate-squared-expected} only applies to bi-Lipschitz activations $\g'$, it has the advantage that the assumption it makes on $p$ corresponds to a convex optimization problem and, when the marginal distribution has certain concentration properties (for generalization), can be solved efficiently through gradient descent on the empirical loss function. As a consequence, for bi-Lipschitz activations we can obtain $O(\opt)$ efficiently under mild distributional assumptions in the agnostic setting.

\section{Agnostically Learning Single-Index Models}\label{section:learning-sims}

In this section, we provide our main result on agnostically learning SIMs. We combine the distortion bounds we established in Section \ref{section:distortion-bounds} with results from \cite{GopalanHKRW23} on omniprediction, which can be used to learn a predictor $p$ that satisfies the assumption of Theorem \ref{corollary:linear-surrogate-squared-expected} simultaneously for all bi-Lipschitz activations. By doing so, we obtain a result for all Lipschitz and non-decreasing activations simultaneously.

\begin{theorem}[Agnostically Learning SIMs]\label{theorem:main}
    Let $\Djoint$ be a distribution over $\R^d\times [0,1]$ with second moments bounded by $\dparam$. There is an algorithm that agnostically learns the class $\simclass_B$ over $\Djoint$ up to $\ell_2$ error $O(\wnorm \sqrt{\dparam}\sqrt{\opt(\simclass_\wnorm, D)}) + \epsilon$ using time and sample complexity $\poly(d, \wnorm, \dparam, \frac{1}{\epsilon})$.
\end{theorem}

In order to apply Theorem \ref{corollary:linear-surrogate-squared-expected}, we use the following theorem which is a combination of results in \cite{GopalanHKRW23}, where they show that the matching losses corresponding to a wide class of functions can all be minimized simultaneously by an efficiently computable predictor.

\begin{theorem}[Omnipredictors for Matching Losses, combination of results in \cite{GopalanHKRW23}]\label{theorem:omnipredictors}
    Let $\Djoint$ be a distribution over $\R^d\times [0,1]$ whose marginal on $\R^d$ has $\dparam$-bounded second moments. There is an algorithm that, given sample access to $\Djoint$, with high probability returns a predictor $p:\R \to (0,1)$ with the following guarantee. For any Fenchel-Legendre pair $(\f,\g)$ such that $\g':\R\to\R$ is $L$-Lipschitz, and $\f'$ satisfies some mild boundedness conditions (see Definition~\ref{definition:bounded-functions}), $p$ satisfies
    \[
        \L_\g(\f'\circ p \;;\Djoint) \le \min_{\|\w\|_2 \le \wnorm}\L_\g(\w \,;\Djoint) + \epsilon.
    \]
    The algorithm requires time and sample complexity $\poly(\dparam, \wnorm, L, \frac{1}{\epsilon})$.
    
\end{theorem}

We aim to apply Theorem \ref{theorem:omnipredictors} to the class of all Lipschitz activations (which is wider than the class of bi-Lipschitz activations). This is enabled by the following lemma, whose proof is based on Theorem \ref{corollary:linear-surrogate-squared-expected} and the fact that the error of a predictor is bounded by the sum of the error of another predictor and the squared expected distance between the two predictors.

\begin{lemma}\label{lemma:general-activations}
    Let $\Djoint$ be a distribution over $\R^d\times [0,1]$. Let $\g':\R\to \R$ be some fixed activation, and $f'$ its dual. Consider the class $\glmclass_{g', \wnorm}$, and let $\w^*$ be the weights achieving $\opt(\glmclass_{g', \wnorm}, \Djoint)$. Let $\phi':\R\to \R$ be an $[\alpha,\beta]$ bi-Lipschitz function (differentiable except possibly at finitely many points) that we wish to approximate $g'$ by. Any predictor $p:\R^d\to \R$ that satisfies 
    \[
        \L_{\phi}(\f'\circ p\,;\Djoint) \le \min_{\|\w\|_2\le \wnorm}\L_{\phi}(\w\,;\Djoint) + \epsilon
    \]
    also satisfies the following $\ell_2$ error guarantee:
    \[
        \elltwoerror(p) \le \frac{2\beta}{\alpha} \, \opt(\glmclass_{g', \wnorm}) + \frac{2\beta}{\alpha} \, \ex\left[ \Bigr(\g'(\w^*\cdot \x) - \phi'(\w^*\cdot \x)\Bigr)^2 \right] + 2\beta\epsilon.
    \]
\end{lemma}

By combining Theorem \ref{theorem:omnipredictors} with Lemma \ref{lemma:general-activations} (whose proofs can be found in Appendix \ref{appendix:proofs-of-section-learning-sims}), we are now ready to prove our main theorem.

\begin{proof}[Proof of Theorem \ref{theorem:main}]
    We will combine Theorem \ref{theorem:omnipredictors}, which states that there is an efficient algorithm that simultaneously minimizes the matching losses corresponding to bounded, non-decreasing and Lipschitz activations, with Lemma \ref{lemma:general-activations}, which implies that minimizing the matching loss corresponding to the nearest bi-Lipschitz activation is sufficient to obtain small error. Note that we may assume that $\epsilon<1/2$, since otherwise the problem is trivial (output the zero function and pick $C = 2$).

    As a first step, we show that link functions corresponding to bi-Lipschitz activations are bounded (according to Definition \ref{definition:bounded-functions}, for $\speedparam = 0$). In particular, let $\phi':\R\to\R$ be an $[\alpha,\beta]$ bi-Lipschitz activation for some $\beta\ge \alpha>0$ such that $\phi'(s) \in [-1,2]$ for some $s\in\R$ and let $\psi'$ be the inverse of $\phi'$ ($\phi'$ is invertible since it is strictly increasing). We will show that $\psi'(r)\in[-R,R]$ for any $r\in[0,1]$, for some $R = O(|s|+1/\alpha)$. 
    
    We pick $r_0=0$, $r_1=1$ and have that $|\psi'(\phi'(s))-\psi'(r_0)| \le \frac{1}{\alpha}|\phi'(s)-r_0| \le \frac{2}{\alpha}$. Hence $\psi'(r_0) \ge \psi'(\phi'(s)) - \frac{2}{\alpha} = s - \frac{1}{\alpha}$. Similarly, we have $\psi'(r_1)\le s+\frac{2}{\alpha}$. Therefore, $\psi'(r) \in [\psi'(0),\psi'(1)]\subseteq [-|s|-\frac{2}{\alpha},|s|+\frac{2}{\alpha}]$, for any $r\in[0,1]$, due to monotonicity of $\psi'$.

    For a given non-decreasing and $1$-Lipschitz $\g'$, we will now show that there is a bounded bi-Lipschitz activation $\phi'$ such that if the assumption of Lemma \ref{lemma:general-activations} is satisfied for $\phi'$ by a predictor $p$, then the error of $p$ is bounded by
    \[
        \elltwoerror(p) \le O(\wnorm\sqrt{\dparam}\sqrt{\opt_\g}) + O(\dparam \wnorm^2 \epsilon)
    \]
    Suppose, first, that $\opt_\g \le \epsilon^2$. Then, we pick $\phi'(t) = \g'(t)+\epsilon t$. Note that $\phi'$ is $[\epsilon,1+\epsilon]$ bi-Lipschitz. Moreover, since $\opt_\g\le \epsilon^2$, we must have some $s\in\R$ with $|s|\le 2{\dparam \wnorm^2}$ such that $\g'(s)\in[-1,2]$. Otherwise, $\opt_\g \ge \pr[|\w^*\cdot \x|\le 2{\dparam \wnorm^2}] = 1 - \pr[|\w^*\cdot \x|> 2{\dparam \wnorm^2}] \ge \frac{1}{4} >\epsilon^2$, due to Chebyshev's inequality, the fact that $\w^*\in\ws$ and the bounded moments assumption. Therefore, $\psi'$ is $(\fbound=2\dparam\wnorm^2+\frac{2}{\epsilon},\speedparam=0)$-bounded and we have
    \[
        \ex\left[ \Bigr( \g'(\w^*\cdot \x) - \phi'(\w^*\cdot \x) \Bigr)^2 \right] \le \epsilon^2\,\ex[(\w^*\cdot \x)^2] \le \epsilon^2 \dparam\wnorm^2
    \]
    As a consequence, under the assumption of Lemma \ref{lemma:general-activations} for $\phi'$, the error of the corresponding predictor $p$ is $\elltwoerror(p)\le 2(1+\epsilon) \epsilon + 2(1+\epsilon) \epsilon \dparam\wnorm^2 + 2(1+\epsilon)\epsilon = O(\dparam \wnorm^2 \epsilon)$.

    In the case that $\opt_\g>\epsilon^2$, we pick $\phi'(t) = \g'(t) + t\;{\frac{\sqrt{\opt_\g}}{\wnorm\sqrt{\dparam}}}$. We may also assume that $\opt_\g \le 1/2$, since otherwise any predictor with range $[0,1]$ will have error at most $2\opt_\g$. Then, $\phi'$ is $[\frac{1}{\wnorm}\sqrt{\opt_\g}/\sqrt{\dparam}, 1+\frac{1}{B}]$ bi-Lipschitz which gives
    \[
        \ex\left[ \Bigr( \g'(\w^*\cdot \x) - \phi'(\w^*\cdot \x) \Bigr)^2 \right] \le \frac{\opt_\g}{\wnorm^2\dparam}\,\ex[(\w^*\cdot \x)^2] \le \opt_\g 
    \]
    As a consequence, under the assumption of Lemma \ref{lemma:general-activations} for $\phi'$, the error of the corresponding predictor $p$ is $\elltwoerror(p)\le 4(1+\frac{1}{\wnorm})\wnorm\sqrt{\dparam}\sqrt{\opt_\g} + 2(1+\frac{1}{\wnorm})\epsilon$. Using a similar approach as for the case $\opt_\g\le \epsilon$, we can show that $\psi'$ is polynomially bounded (as per Definition \ref{definition:bounded-functions}), since $\opt_\g \le \frac{1}{2}$.

    To conclude the proof of Theorem \ref{theorem:main}, we apply Theorem \ref{theorem:omnipredictors} with appropriate (polynomial) choice of parameters, to show that there is an efficient algorithm that outputs a predictor $p:\R^d\to(0,1)$ for which the assumption of Lemma \ref{lemma:general-activations} holds simultaneously for all bi-Lipschitz activations ($\phi'$) with sufficiently bounded inverses ($\psi'$).
\end{proof}

\section{Stronger Guarantees for Logistic Regression}\label{section:stronger-bounds}

In this section, we follow the same recipe we used in Section \ref{section:distortion-bounds} to obtain distortion bounds similar to Theorem \ref{corollary:linear-surrogate-squared-expected} for the sigmoid activation (or, equivalently, the logistic model) under the assumption that the marginal distribution is sufficiently concentrated (see Definition \ref{definition:bounded-functions}). In particular, Theorem \ref{corollary:linear-surrogate-squared-expected} does not hold, since the sigmoid is not bi-Lipschitz and our main Theorem \ref{theorem:main} only provides a guarantee of $O(\sqrt{\opt})$ for squared error. We use appropriate pointwise distortion bounds for the matching loss corresponding to the sigmoid activation and provide guarantees of $\widetilde{O}({\opt})$ for logistic regression with respect to both squared and absolute error, under appropriate assumptions about the concentration of the marginal distribution. The proofs of this section are provided in Appendix \ref{appendix:proofs-of-section-stronger-bounds}.

For the logistic model, the link function $\f'$ is defined as $\f'(r) = \ln(\frac{r}{1-r})$, for $r\in(0,1)$ and the corresponding activation $\g'$ is the sigmoid $\g'(t) = \frac{1}{1+e^{-t}}$ for $t\in \R$. The corresponding matching loss is the logistic loss.

\paragraph{Squared error.} We first provide a result for squared loss minimization. In comparison to Theorem \ref{corollary:linear-surrogate-squared-expected}, the qualitative interpretation of the following theorem is that, while the sigmoid activation is not bi-Lipschitz, it is effectively bi-Lipschitz under sufficiently concentrated marginals.

\begin{theorem}[Squared Loss Minimization through Logistic Loss Minimization]\label{theorem:logistic-squared-expected}
    Let $\Djoint$ be a distribution over $\R^d\times [0,1]$ whose marginal on $\R^d$ is $(1,2)$-concentrated. Let $\g':\R\to\R$ be the sigmoid activation, i.e., $\g'(t) = (1+e^{-t})^{-1}$ for $t\in \R$. Assume that for some $\wnorm>0$, $\epsilon>0$ and a predictor $p:\R^d \to (0,1)$ we have
    \begin{equation}\label{equation:assumption-theorem-surrogate-squared-expected}
        \L_{\g}(\f'\circ p\,;\Djoint) \le \min_{\w: \|\w\|_2\le \wnorm}\L_{\g}(\w\,;\Djoint) + \epsilon
    \end{equation}
    If we let $\opt_\g = \min_{\|\w\|_2\le \wnorm} \elltwoerror(\g'_{\w})$, then for the predictor $p$ and some universal constant $C>0$ we also have
    \[
        \elltwoerror(p) \le C\,\opt_\g \exp\left(\wnorm^2 + {\sqrt{\wnorm^2\,\log\frac{1}{\opt_\g}}}\right) + 2\epsilon.
    \]
\end{theorem}
In particular, the squared error of $p$ is upper bounded by $\widetilde{O}(\opt_\g)$, since the function $t\mapsto\exp(\log^{1/2} t)$ is asymptotically smaller than any polynomial function $t\mapsto t^{\gamma}$ with $\gamma >0$.

Once more, the proof of our result is based on an appropriate pointwise distortion bound which we provide below. It follows by the fact that the Bregman divergence corresponding to the sigmoid is the Kullback-Leibler divergence and by combining Pinsker's inequality (lower bound) with Lemma 4.1 of \cite{gotze2019higher} (upper bound).

\begin{lemma}[Pointwise Distortion Bound for Sigmoid]\label{lemma:kl-and-squared-loss}
    Let $g'$ be the sigmoid activation. Then, for any $y,p\in(0,1)$ we have that $\Bregman_f(y,p) = \kl(y\| p) = y\ln(y/p) + (1-y)\ln(\frac{1-y}{1-p})$. Moreover
    \[
        \ell_{\g}(y,{\f'}(p)) - \ell_{\g}(y,{\f'}(y)) = \kl(y\| p) \in \left[ \frac{1}{2}(y-p)^2, \; \frac{2}{\min\{p,1-p\} }\cdot (y-p)^2 \right]
    \]
\end{lemma}
We translate Lemma \ref{lemma:kl-and-squared-loss} to a bound on the squared error of a matching loss minimizer (in this case, the logistic loss) using an approach similar to the one for Theorem \ref{corollary:linear-surrogate-squared-expected}. In order to use the upper bound on the surrogate loss provided by Lemma \ref{lemma:kl-and-squared-loss}, we apply it to $p \leftarrow \g'(\w\cdot \x)$, where $\g'$ is the sigmoid function, and observe that the quantity $\frac{1}{p(1-p)}$ is exponential in $|\w\cdot \x|$. Hence, when the marginal is $(\dparam,2)$-concentrated (subgaussian concentration), then $\frac{1}{p(1-p)}$ is effectively bounded.

\paragraph{Absolute error.} All of the results we have provided so far have focused on squared error minimization. We now show that our approach yields results even for the absolute error, which can also be viewed as learning in the p-concept model \cite{kearns1994efficient}. In particular, for a distribution $\Djoint$ over $\R^d\times[0,1]$, we define the absolute error of a predictor $p:\R^d\to [0,1]$ as follows.
\[
    \elloneerror(p) = \ex_{(\x,y)\sim\Djoint}[|y-p(\x)|]
\]
In the specific case when the labels are binary, i.e., $y\in\{0,1\}$, we have
\[
    \elloneerror(p) = \ex_{(\x,y)\sim\Djoint}[|y-p(\x)|] = \pr_{(\x,y,y_p) \sim \Djoint_p}[y\neq y_p] \tag{ see Proposition \ref{proposition:ell-1-p-concept}}
\]
where the distribution $\Djoint_p$ is over $\R^d\times\{0,1\}\times\{0,1\}$ and is formed by drawing samples $(\x,y)$ from $\Djoint$ and, given $\x$, forming $y_p$ by drawing a conditionally independent Bernoulli random variable with parameter $p(\x)$. We provide the following result.

\begin{theorem}[Absolute Loss Minimization through Logistic Loss Minimization]\label{theorem:logistic-absolute-expected}
    Let $\Djoint$ be a distribution over $\R^d\times \{0,1\}$ whose marginal on $\R^d$ is $(1,1)$-concentrated. Let $\g':\R\to\R$ be the sigmoid activation, i.e., $\g'(t) = (1+e^{-t})^{-1}$ for $t\in \R$. Assume that for some $\wnorm>0$, $\epsilon>0$ and a predictor $p:\R^d \to (0,1)$ we have
    \begin{equation}\label{equation:assumption-theorem-surrogate-absolute-expected}
        \L_{\g}(\f'\circ p\,;\Djoint) \le \min_{\w: \|\w\|_2\le \wnorm}\L_{\g}(\w\,;\Djoint) + \epsilon
    \end{equation}
    If we let $\opt_\g = \min_{\|\w\|_2\le \wnorm} \elloneerror(\g'_{\w})$, then for the predictor $p$ and some universal constant $C>0$ we also have
    \[
        \elloneerror(p) \le C\,\wnorm \,\opt_\g\, \log\frac{1}{\opt_\g} + \epsilon
    \]
\end{theorem}

The corresponding distortion bound in this case is between the absolute and logistic losses and works when the labels are binary.

\begin{lemma}[Pointwise Distortion between Absolute and Logistic Loss]\label{lemma:kl-and-absolute-loss}
    Let $g'$ be the sigmoid activation. Then, there is a constant $c\in \R$ such that for any $y\in \{0,1\}$ and $p\in(0,1)$, we have
    \[
        \ell_{\g}(y,{\f'}(p)) - c \in \left[ |y-p| \;,\;    2 \cdot \ln\left(\frac{1}{p(1-p) } \right) \cdot |y-p| \right]
    \]
\end{lemma}

The bound of Theorem \ref{theorem:logistic-absolute-expected} implies an algorithm for learning an unknown sigmoid neuron in the p-concept model, by minimizing a convex loss. While there are algorithms achieving stronger guarantees \cite{diakonikolas22learningsingleneuron} for agnostically learning sigmoid neurons, such algorithms typically make strong distributional assumptions including concentration, anti-concentration and anti-anti-concentration or boundedness.

Moreover, it is useful to compare the bound we provide in Theorem \ref{theorem:logistic-absolute-expected} to a lower bound by \cite[Theorem 4.1]{diakonikolas20non}, which concerns the problem of agnostically learning halfspaces by minimizing convex surrogates. In particular, they show that even under log-concave marginals, no convex surrogate loss can achieve a guarantee better than $O(\opt \log(1/\opt))$, where $\opt$ is measured with respect to the $\ell_1$ error (which is equal to the probability of error). The result is not directly comparable to our upper bound, since we examine the sigmoid activation. Their setting can be viewed as a limit case of ours by letting the norm of the weight vector grow indefinitely (the sigmoid tends to the step function), but the main complication is that our upper bound is of the form $O(\wnorm \opt \log(1/\opt))$, which scales with $\wnorm$. However, their lower bound concerns marginal distributions that are not only concentrated, but are also anti-concentrated and anti-anticoncentrated, while our results only make concentration assumptions.

\section{Necessity of Norm Dependence}\label{sec:lower-bound}

In this final section, we use a lower bound due to \cite{{diakonikolas22hardness}} on agnostic learning of GLMs  using SQ algorithms and compare it with our main result (Theorem \ref{theorem:main}). For simplicity, we specialize to the case of the standard sigmoid or logistic function. A modification of their proof ensures that the bound holds under isotropic marginals.\footnote{Specifically, our features correspond to all multilinear monomials (or parities) of degree at most $k$ over $\{\pm 1\}^n$, whereas they use all monomials (not necessarily multilinear) of degree at most $k$. These yield equivalent representations since the hard distributions are obtained from the uniform distribution on $\{\pm 1\}^n$.}
\begin{theorem}[{SQ Lower Bound for Agnostically Learning GLMs, variant of \cite[Thm C.3]{diakonikolas22hardness}}]\label{theorem:sq-lower-bound}
    Let $g' : \R^d \to \R$ be the standard logistic function. Any SQ algorithm either requires $d^{\omega(1)}$ queries or $d^{-\omega(1)}$ tolerance to distinguish between the following two labeled distributions: \begin{itemize}
        \item (Labels have signal.) $\Djoint_{\signal}$ on $\R^d \times \R$ is such that $\opt(\glmclass_{g', B}, \Djoint_\signal) \leq \exp(-\Omega(\log^{1/4} d)) = o(1)$ for some $B = \poly(d)$.
        \item (Labels are random.) $\Djoint_{\random}$ on $\R^d \times \R$ is such that the labels $y$ are drawn i.i.d. from $\{a, b\}$ for certain universal constants $a, b \in [0, 1]$. In particular, $\opt(\glmclass_{g', B}, \Djoint_{\random}) = \Omega(1)$ for any $B$.
    \end{itemize} Both $\Djoint_{\signal}$ and $\Djoint_{\random}$ have the same marginal on $\R^d$, with $1$-bounded second moments.
\end{theorem}

Let us consider applying our main theorem (Theorem~\ref{theorem:main}) to this setting, with $\Djoint$ being either $\Djoint_{\signal}$ or $\Djoint_{\random}$, and with the same $\wnorm = \poly(d)$ as is required to achieve small error in the ``labels have signal'' case. We would obtain a predictor with $\ell_2$ error at most $B \sqrt{\opt(\glmclass_{g', B})}$ (or indeed with $\simclass_{B}$ in place of $\glmclass_{g', B}$). Since this is $\omega(1)$, this guarantee is insufficient to distinguish the two cases above, which is as it should be since our main algorithm indeed fits into the SQ framework.

Theorem~\ref{theorem:sq-lower-bound} does, however, justify a dependence on the norm $B$ in our main result. In particular, it is clear that a guarantee of the form $\opt(\glmclass_{g', B})^{c}$ for any universal constant $c > 0$ (independent of $\wnorm$) would be too strong, as it would let us distinguish the two cases above. In fact, this lower bound rules out a large space of potential error guarantees stated as functions of $B$ and $\opt(\glmclass_{g', B})$.  For instance, for sufficiently large $d$, it rules out any error guarantee of the form $\exp(O(\log^{1/5} B)) \cdot  \opt(\glmclass_{g', B})^{c'}$ for any universal constant $c' > 0$.

\bibliographystyle{alpha}
\bibliography{references.bib}

\newpage
\appendix

\section{Technical Lemmas}

In this section, we provide some technical Lemmas that we use in our proofs.

\begin{proposition}[Weak Learner for Linear Functions]\label{proposition:weak-learner}
    Let $\Djoint$ be a distribution over $\R^d\times [-1,1]$ whose marginal on $\R^d$ has $\dparam$-bounded second moments and $\wnorm > 0$. For any $\epsilon > 0$ and $\delta\in(0,1)$, there is a universal constant $C>0$ and an algorithm that given a set $S$ of i.i.d. samples from $\Djoint$ of size at least $C\cdot \frac{d^2\dparam\wnorm^2}{\epsilon^2}\log\frac{1}{\delta}$, runs in time $\poly(d,|S|)$ and satisfies the following specifications with probability at least $1-\delta$
    \begin{enumerate}
        \item If $\ex_{(\x,z)\sim\Djoint}[z(\w\cdot \x)] \ge \epsilon$ for some $\w\in\R^d$ with $\|\w\|_2\le \wnorm$, then the algorithm accepts. Otherwise, it may or may not reject and return a special symbol.
        \item If the algorithm accepts then it returns $\w\in \R^d$ with $\|\w\|_2\le \wnorm$ such that we have $\ex_{(\x,z)\sim\Djoint}[z(\w\cdot \x)] \ge \epsilon/4$.
    \end{enumerate}
\end{proposition}

\begin{proof}
    We will prove the proposition for $\delta = 1/6$. We may boost the probability of success with repetition.
    
    The algorithm computes the vector $\vv = \ex_S[z\,\x]$. If $\|\vv\|_2 \le \frac{3\epsilon}{4\wnorm}$, then the algorithm rejects and outputs a special symbol. Otherwise, it outputs the vector $\frac{\wnorm}{\|\vv\|_2} \, \vv$.

    Suppose, first, that $\ex_{(\x,z)\sim\Djoint}[z(\w\cdot \x)] \ge \epsilon$ for some $\w$ with $\|\w\|_2 \le \wnorm$. Then, due to Chebyshev's inequality we have for any $i\in[d]$
    \begin{align*}
        \pr\left[ \left|\ex_S[z\, \x_i] - \ex_{\Djoint}[z\, \x_i]\right| > \frac{\epsilon}{8\, \wnorm\sqrt{d}} \right] \le \frac{64\, d\, \wnorm^2}{|S|\, \epsilon^2} \ex[\x_i^2] \le \frac{64\, d\, \wnorm^2\, \dparam}{|S|\, \epsilon^2} \le \frac{1}{6\, d} \tag{ for large enough $|S|$}
    \end{align*}
    Hence, with probability at least $5/6$, we have $\|\ex_S[z\x] - \ex_\Djoint[z\x]\|_2 \le \frac{\epsilon}{8\wnorm}$, due to a union bound. Therefore, $\|\vv\|_2 \ge \|\ex_\Djoint[z\x]\|_2 - \frac{\epsilon}{8\wnorm} \ge \frac{\ex_{\Djoint}[z(\w\cdot\x)]}{B} - \frac{\epsilon}{8\wnorm} \ge \frac{7\wnorm}{8}$ and the algorithm accepts.

    Suppose, now, that the algorithm accepts. Then, we have $\|\vv\|_2 > \frac{3\epsilon}{4\wnorm}$ and (with probability at least $5/6$) we have
    \[
        \ex_{\Djoint}\left[ \frac{\wnorm}{\|\vv\|_2}\, z(\vv\cdot \x) \right] = \frac{\wnorm}{\|\vv\|_2} \vv\cdot \ex[z\x] \ge \epsilon/4
    \]
    since $\|\ex_S[z\x] - \ex_\Djoint[z\x]\|_2 \le \frac{\epsilon}{8\wnorm}$. This concludes the proof.
\end{proof}

\begin{proposition}\label{proposition:ell-1-p-concept}
    Let $\Djoint$ be a distribution over $\R^d\times \{0,1\}$ and $p:\R^d\to [0,1]$. Consider the distribution $\Djoint_p$ over $\R^d\times\{0,1\}\times\{0,1\}$, which is formed by drawing samples $(\x,y)$ from $\Djoint$ and, given $\x$, forming $y_p$ by drawing a conditionally independent Bernoulli random variable with parameter $p(\x)$. Then we have
    \[
        \elloneerror(p) = \ex_{(\x,y)\sim\Djoint}[|y-p(\x)|] = \pr_{(\x,y,y_p) \sim \Djoint_p}[y\neq y_p]
    \]
\end{proposition}

\begin{proof}
    Since over $\Djoint_p$, $y$ and $y_p$ are conditionally independendent, we have 
\begin{align*}
    \elloneerror(p) = \ex[|y-p(\x)|] &= \ex[(1-p(\x))\ind\{y=1\} + p(\x)\ind\{y=0\}]\\
    &=\ex_\x\Bigr[\pr[y_p = 0 | \x] \pr[y=1 | \x] + \pr[y_p = 1 | \x] \pr[y=0 | \x]\Bigr]\\
    &=\ex_x[\pr[y\neq y_p | \x]] = \pr[y\neq y_p]
\end{align*}
\end{proof}

\section{Proofs from Section \ref{section:distortion-bounds}}

\subsection{Proof of Theorem \ref{corollary:linear-surrogate-squared-expected}}\label{appendix:proof-corollary-linear-surrogate-squared-expected}

To prove Theorem \ref{corollary:linear-surrogate-squared-expected}, we first prove the following more general theorem. Theorem \ref{corollary:linear-surrogate-squared-expected} may then be easily recovered from this by setting $\mH = \glmclass_{\g', \wnorm}$ and observing that $\f'(\g'(\w\cdot \x)) = \w\cdot \x$, since $\g'$ is invertible.

\begin{theorem}[Squared Error Minimization through Distorted Matching Loss Minimization]\label{theorem:surrogate-squared-expected}
    Let $\Djoint$ be a distribution over $\R^d\times [0,1]$, let $0<\alpha\le \beta$ and let $(\f,\g)$ be a pair of Fenchel-Legendre dual functions such that $g':\R\to\R$ is continuous, non-decreasing and $\f':\range(\g')\to \R$ is $[\frac{1}{\beta},\frac{1}{\alpha}]$ bi-Lipschitz. Let $\epsilon>0$ and $\mH\subseteq \{\R^d\to \range(\g')\}$. Assume that for a predictor $p:\R^d \to \range(\g')$ we have
    \begin{equation}\label{equation:assumption-theorem-surrogate-squared-expected-appendix}
        \L_{\g}(\f'\circ p\,;\Djoint) \le \min_{h\in \mH}\L_{\g}(\f'\circ h\,;\Djoint) + \epsilon
    \end{equation}
    Then, for the predictor $p$, we also have: $\elltwoerror(p) \le \frac{\beta}{\alpha} \cdot \min_{h\in\mH} \elltwoerror(h) + 2\beta \epsilon$. 
\end{theorem}

\begin{proof} We apply Lemma \ref{lemma:surrogate-squared-pointwise} with $y\leftarrow y$ and $p\leftarrow p(\x)$ and take expectations over $\Djoint$ on both sides. We have that
\[
    \elltwoerror(p) \le 2\beta\cdot \ex \ell_\g(y, \f'(p(\x)) ) - 2\beta\cdot  \ex \ell_{\g}(y, {\f'}(y))
\]

Therefore, we can bound the squared error of $p$ as follows.
\begin{align*}
    \elltwoerror(p) &\le 2\beta \cdot \L_{\g}(\f'\circ p\,;\Djoint) - 2\beta\cdot Q^* \\
    &\le 2\beta \cdot \L_{\g}(\f'\circ h\,;\Djoint) - 2\beta\cdot Q^* \tag{by assumption, for any $h\in\mH$}
\end{align*}
where $Q^* = \ex \ell_\g(y, \f'(y))$.

We now apply Lemma \ref{lemma:surrogate-squared-pointwise} again with $y\leftarrow y$ and $p\leftarrow h(\x)$ and we similarly have
\begin{align*}
    \ex \ell_\g(y, \f'\circ h(\x) ) - Q^* &\le \frac{1}{2\alpha}\elltwoerror(h)
\end{align*}
Therefore, for any $h\in\mH$, we have, in total: $\elltwoerror(p) \le \frac{\beta}{\alpha} \elltwoerror(h) + 2\beta \epsilon$.
\end{proof}

We first prove Lemma \ref{lemma:surrogate-squared-pointwise}, which we restate here for convenience.

\begin{lemma}\label{lemma:surrogate-squared-pointwise-appendix}
    Assume ${f'}$ is $[1/\beta,1/\alpha]$ bi-Lipschitz and differentiable on all except from a finite number of points on any bounded interval. Then for any $y, p\in\range(g')$ we have
    \[
        \ell_{g}(y,{f'}(p)) - \ell_{g}(y,{f'}(y)) = D_{f}(y,p) \in \left[ \frac{1}{2\beta}(y-p)^2 , \frac{1}{2\alpha} (y-p)^2 \right]
    \]
\end{lemma}

\begin{proof}
    We first show that $\ell_{g}(y,{f'}(p)) - \ell_{g}(y,{f'}(y)) = D_{f}(y,p)$. In particular, we have
    \[
        g(f'(p)) = f'(p) g'(f'(p)) - f(g'(f'(p))) = pf'(p) - f(p)\,,
    \]
    since $f'(p)\in \range(f')$ and we know that $g(t) = t g'(t) - f(g'(t))$ for any $t\in\range(f')$ as well as $g'(f'(p)) = p$ for any $p\in\range(g')$. Therefore, we have
    \begin{align*}
        \ell_{g}(y,{f'}(p)) - \ell_{g}(y,{f'}(y)) &= g(f'(p)) - yf'(p) - g(f'(y)) + y f'(y)\\
        &= pf'(p) - f(p) - yf'(p) -yf'(y) + f(y) + y f'(y) = \\
        &= f(y) - f(p) - (y-p) f'(p) = D_{f}(y,p)\,.
    \end{align*}

    Let $\psi:\range(g')\to \R$ be such that $\psi'(p) = f'(p)$ and $\psi'$ is differentiable on the open interval between $y$ and $p$, with $\psi''(\xi) \in [1/\beta, 1/\alpha]$ for any $\xi$ between $y$ and $p$.
    Let $\gamma_y:= f(y) - \psi(y), \gamma_p := f(p) - \psi(p)$ and $\gamma_\psi := 2\max\{|\gamma_y|, |\gamma_p|\}$. Then we have that
    \[
        D_{f}(y,p) = \psi(y) - \psi(p) - (y-p) \psi'(p) + (\gamma_y-\gamma_p) =  \frac{1}{2} \psi''(\xi) (y-p)^2+(\gamma_y-\gamma_p)
    \]
    \[
        D_{f}(y,p) \in \left[\frac{1}{2\beta}  (y-p)^2 - 2\gamma_\psi , \frac{1}{2\alpha}  (y-p)^2 + 2\gamma_\psi \right]\,,
    \]
    for any $\psi$ as defined above (say $\psi\in\Psi$). In particular, we have
    \[
        D_{f}(y,p) \le \frac{1}{2\alpha}  (y-p)^2 + 2 \inf_{\psi\in\Psi}\gamma_\psi \text{ and }
    \]
    \[
        D_{f}(y,p) \ge \frac{1}{2\beta}  (y-p)^2 - 2 \inf_{\psi\in\Psi}\gamma_\psi
    \]
    Since, we have only a finite number of points where the derivative is not well defined, a simple smoothening technique may give us $\Psi$ such that $\inf_{\psi\in\Psi}\gamma_\psi = 0$.
\end{proof}

\section{Proofs from Section \ref{section:learning-sims}}\label{appendix:proofs-of-section-learning-sims}

\subsection{Proof of Theorem \ref{theorem:omnipredictors}}

We first define a boundedness property which we use in order to apply the results from \cite{GopalanHKRW23}. The property states that the activation function (the partial inverse of the link function) must either have a range that covers all possible labels, or has a range whose closure covers all possible labels and the rate with which the labels are covered as we tend to the limits of the domain is at least polynomial. For example, the sigmoid activation tends to $1$ (resp. $0$) exponentially fast as its argument increases (resp. decreases).

\begin{definition}[Bounded Functions]\label{definition:bounded-functions}
    Let $u:(0,1)\to \R$ be a non-decreasing function defined on the interval $(0,1)$. For $\fbound,\speedparam\ge 0$, we say that $u$ is $(\fbound,\speedparam)$-bounded on $[0,1]$ if for any $\epsilon>0$, there are $r_0\le r_1\in[0,1]$ such that if we let $u(r_i) = \lim_{r\to r_i} u(r),$ $i=0,1$ then
    \begin{align*}
        \max\{-u(r_0), u(r_1)\} &\le \fbound\left(\frac{1}{\epsilon}\right)^{\speedparam}  \\
        (1-r_1)(u(r) - u(r_1)) &\le \epsilon \text{ for } r\ge r_1 \text{ and} \\
        r_0 (u(r_0) - u(r)) &\le \epsilon \text{ for } r\le r_0
    \end{align*}
\end{definition}

We restate a slightly more quantitative version of Theorem \ref{theorem:omnipredictors} here for convenience.

\begin{theorem}[Omnipredictors for Matching Losses, combination of results in \cite{GopalanHKRW23}]\label{theorem:omnipredictors-appendix}
    Let $\Djoint$ be a distribution over $\R^d\times [0,1]$ whose marginal on $\R^d$ has $\dparam$-bounded second moments. There is an algorithm that, given sample access to $\Djoint$, with high probability returns a predictor $p:\R \to (0,1)$ with the following guarantee. For any pair of Fenchel-Legendre dual functions $(\f,\g)$ such that $\g':\R\to\R$ is continuous, non-decreasing and $L$-Lipschitz, and $\f'$ is $(\fbound,\speedparam)$-bounded (see Definition~\ref{definition:bounded-functions}), $p$ satisfies
    \[
        \L_\g(\f'\circ p \;;\Djoint) \le \min_{\|\w\|_2 \le \wnorm}\L_\g(\w \,;\Djoint) + \epsilon.
    \]
    The algorithm requires time and sample complexity $\poly(\dparam, \wnorm, L, \fbound, \frac{1}{\epsilon}, \frac{1}{\epsilon^\speedparam})$.
    
\end{theorem}

\begin{proof}[Proof of Theorem \ref{theorem:omnipredictors}]
    Suppose first that the marginal of $\Djoint$ on $\R^d$ is supported on the unit ball $\Ball_d$ and that the labels are binary. Then, the result would follow from Theorems 7.7 and A.4 of \cite{GopalanHKRW23}. In particular, Theorem 7.7 states that given access to a weak learner with the specifications of Proposition \ref{proposition:weak-learner}, there is an efficient algorithm that computes an $\epsilon_1$-calibrated and $(\C,\epsilon_1)$-multiaccurate predictor $p$, where the notions of calibration and multiaccuracy originate to the literature of fairness and are defined, e.g., in Definitions 3.1 and 3.2 of \cite{GopalanHKRW23} and $\C = \{\x\to \w\cdot \x\;|\; \|\w\|_2\le \wnorm\}\cup\{\x\to 1\}$ (the class $\C$ is bounded as long as $\|\x\|_2\le 1$ almost surely). Theorem A.4 states that for $\epsilon_2>0$, any $\epsilon_1$-calibrated and $(\C,\epsilon_1)$-multiaccurate predictor $p$ minimizes simultaneously the matching loss corresponding to any pair $(\f,\g)\in\fenchelpairs$ (where $\f$ is $(\fbound,\speedparam)$-bounded) up to error 
    \[
        \fbound(1/\epsilon_2)^\speedparam \epsilon_1 + \epsilon_1 + \epsilon_2
    \]   
    The expression above is formed by proving that any pair $(\f,\g)\in \fenchelpairs$ has the property that $\f'$ is $(\epsilon_2,\fbound(1/\epsilon_2)^\speedparam)$-approximately optimal (as per the Definition A.1 of \cite{GopalanHKRW23}), for any $\epsilon_2>0$. In particular, we would like to show that for any $\epsilon_2>0$ there exists $\hat{f}'$ such that the following is true for any $p\in[0,1]$
    \begin{align*}
        \ell_\g(r, \hat{\f}'(r)) &\le \ell_\g(r,\f'(r)) + \epsilon_2 \\
        |\hat{\f}'(r)| &\le \fbound\cdot(1/\epsilon_2)^\speedparam
    \end{align*}
    We may pick $\hat{\f}'$ as follows (for $r_0\le r_1$ as given by Definition \ref{definition:bounded-functions}).
    \[
        \hat{\f}'(r) = 
        \begin{cases}
            \f'(r), \text{ if }r\in[r_0,r_1]\\
            \f'(r_0), \text{ if }r<r_0 \\
            \f'(r_1), \text{ if }r>r_1
        \end{cases}
    \]
    The desired result follows from using the expression for $\ell_\g$, the convexity of $\g$ (since $\g'$ is non decreasing) and the guarantees of Definition \ref{definition:bounded-functions}.

    However, we only assume that the marginal distribution is not bounded and we, therefore, need to make certain modifications to the proof of their Theorem 7.7. In particular, the boundedness assumption is used in the proofs of Lemma 7.2, Lemma 7.6 and Theorem 7.7 in \cite{GopalanHKRW23}. For Lemma 7.2, the guarantee for $p_2$ changes to $(\C,\alpha + \wnorm\sqrt{\dparam\delta})$-multiaccuracy, by using a Cauchy-Schwarz inequality and bounding $\ex[(\w\cdot \x)^2]$ by $\wnorm^2\cdot \dparam$. Note that $\delta$ here is a parameter they use within their algorithm and it can be picked quadratically smaller (resulting into a polynomial increase of the time and sample complexity of their algorithm). For Lemma 7.6, one needs to pick a step size that is polynomially smaller than the guarantee that the weak learner provides. In particular, within the proof of their Lemma 7.6 in their appendix, if the weak learner of our Proposition \ref{proposition:weak-learner} is run with $\epsilon \leftarrow \epsilon_3$, then one gets (within their proof)
    \begin{align*}
        l_2(p^*,p_t)^2-l_2(p^*,p_{t+1})^2 \ge \frac{\sigma\cdot \epsilon_3}{2} - B^2\dparam \sigma^2
    \end{align*}
    If $\sigma$ (the step size of their Algorithm 1 for multiaccuracy) is picked small enough with respect to $\epsilon_3$ and parameters $B^2,\dparam$, then the quantity of interest $l_2(p^*,p_t)^2-l_2(p^*,p_{t+1})^2$ has a good enough lower bound. Note that in their original algorithm, $\sigma$ was picked equal to $\epsilon_3$ and this is why $\epsilon_3$ (or another corresponding parameter) does not appear in their proofs. Once more, the updated choice of $\sigma$ generates a polynomial overhead in time and sample complexity. The aforementioned modifications are sufficient for the modified version of Theorem 7.7 and its proof.

    The final technical complication we need to address is that their algorithm is guaranteed to work only given binary labels. We can, however, form binary labels as follows. Let $(\x,y)$ be drawn from $\Djoint$. We have that $y\in[0,1]$. We draw a conditionally (on $y$) independent Bernoulli random variable $y'$ with probability of success $y$, forming the distribution $\Djoint'$ over $\R^d\times \{0,1\}$. We run the algorithm of \cite{GopalanHKRW23} on $\Djoint'$ and obtain a predictor $p$ such that 
    \[
        \L_\g(\f'\circ p \;;\Djoint') \le \min_{\|\w\|_2\le \wnorm}\L_\g(\w \,;\Djoint') + \epsilon\,,\, 
    \]
    for any $(\f,\g)$ as described in the theorem statement.
    We, additionally, have that for any $c:\R^d\to \R$ 
    \begin{align*}
        \L_\g(c \,;\Djoint') &= \ex_{\x,y'}[\g(c( \x)) - y'c(\x)] \\
        &= \ex_{\x}\left[\g(c(\x)) - \ex_{y'}\left[y'\,\bigr|\,\x\right]c(\x)\right] \\
        &= \ex_{\x}\left[\g(c(\x)) - \ex_{y}\left[\ex_{y'}\left[y'\bigr|\,y,\x\right]\,\biggr|\,\x\right]c(\x)\right] \\
        &= \ex_{\x}\left[\g(c(\x)) - \ex_{y}\left[\ex_{y'}\left[y'\bigr|\,y\right]\,\biggr|\,\x\right]c(\x)\right] \tag{$y'$ is independent from $\x$ given $y$}\\
        &= \ex_{\x}\left[\g(c(\x)) - \ex_{y}\left[y\,\bigr|\,\x\right]c(\x)\right] \tag{$y'$ is Bernoulli($y$)} \\
        &= \ex_{\x,y}[\g(c(\x)) - yc(\x)] = \L_\g(c\,;\Djoint)
    \end{align*}
    This concludes the proof of Theorem \ref{theorem:omnipredictors}.
\end{proof}

\subsection{Proof of Lemma \ref{lemma:general-activations}}\label{appendix:proof-of-lemma-general-activations}

    We first apply Theorem \ref{corollary:linear-surrogate-squared-expected} with $\g'\leftarrow \phi'$ to obtain that for $\phi'_{\w}(\x)=\phi'(\w\cdot \x)$, we have
    \[
        \elltwoerror(p) \le \frac{\beta}{\alpha}\, \elltwoerror(\phi'_{\w^*}) + 2\beta\epsilon \tag{since inequality holds for $\w\in\ws$}
    \]
    Moreover, we have
    \begin{align*}
        \elltwoerror(\phi'_{\w^*}) &= \ex\left[\Bigr(y-\phi'(\w^*\cdot \x)\Bigr)^2\right] \\
        &= \ex\left[\Bigr(y-\g'(\w^*\cdot \x)+ \g'(\w^*\cdot \x) -\phi'(\w^*\cdot \x)\Bigr)^2\right] \\
        &\le 2\ex\left[\Bigr(y-\g'(\w^*\cdot \x)\Bigr)^2\right] + 2\ex\left[\Bigr(\g'(\w^*\cdot \x) -\phi'(\w^*\cdot \x)\Bigr)^2\right] \\
        & = 2\,\opt_\g + 2\ex\left[\Bigr(\g'(\w^*\cdot \x) -\phi'(\w^*\cdot \x)\Bigr)^2\right]
    \end{align*}
    This concludes the proof of Lemma \ref{lemma:general-activations}.

\section{Proofs from Section \ref{section:stronger-bounds}}\label{appendix:proofs-of-section-stronger-bounds}

\subsection{Proof of Theorem \ref{theorem:logistic-squared-expected}}

In the case we consider here, $\g'$ is the sigmoid activation, i.e., $\g'(t) = (1+e^{-t})^{-1}$ for any $t\in \R$. In particular, the pointwise surrogate loss we consider satisfies 
\[
    \ell_\g(y,\f'(p)) = y\ln\frac{1}{p} + (1-y) \ln\frac{1}{1-p} - \ln 2\,,
\]
for any $y\in[0,1]$ and $p\in(0,1)$. We may extend Lemma \ref{lemma:kl-and-squared-loss} to also capture $y\in\{0,1\}$, by defining $\ell_\g(0,\f'(0)) = \ell_\g(1,\f'(1)) = -\ln 2$ (the inequality would hold under this definition). Hence, following a similar procedure as the one used for proving Theorem \ref{theorem:surrogate-squared-expected}, we obtain the following by applying Lemma \ref{lemma:kl-and-squared-loss}
\begin{align}
    \elltwoerror(p) &\le 2\left(\L_{\g}(\f'\circ p) - \ex[\ell_\g(y,\f'(y))]\right)\label{equation:sigmoid-error-le-surr} \\
    \L_{\g}(\w^*) - \ex[\ell_\g(y,\f'(y))] &\le \ex\left[\frac{2}{\g'(\w^*\cdot \x) \vee (1-\g'(\w^*\cdot \x))} \cdot (y-\g'(\w^*\cdot \x))^2\right]\label{equation:sigmoid-surr-le-distorted-err} \\
    \L_{\g}(\f'\circ p) &\le \L_{\g}(\w^*) + \epsilon
\end{align}

Therefore, in order to prove Theorem \ref{theorem:logistic-squared-expected}, it is sufficient to provide a strong enough upper bound for the quantity of the right hand side of Equation \eqref{equation:sigmoid-surr-le-distorted-err} in terms of $\opt_\g$. We observe that
\[
    \frac{2}{\g'(\w^*\cdot \x)\vee (1-\g'(\w^*\cdot \x))} \le 4\exp({|\w^*\cdot \x|})\,, \text{ for any }\x\in\R^d
\]

It remains to bound the quantity $\ex[ e^{|\w^*\cdot \x|} (y - \g'_{\w}(\x))^2]$. Set $Q = e^{|\w^*\cdot \x|} (y - \g'_{\w}(\x))^2 $ ($Q$ is a random variable). Then for any $r\ge 0$ we have
    \begin{align*}
        \ex[Q] &= \ex[Q\cdot \ind\{|\w^*\cdot \x|\le r\}] + \ex[Q\cdot \ind\{|\w^*\cdot \x|> r\}] \\
        &\le e^r \ex[(y-\g'_{\w^*}(\x))^2\ind\{|\w^*\cdot \x|\le r\}] + \ex[e^{|\w^*\cdot \x|}\ind\{|\w^*\cdot \x|> r\}] \\
        & \le e^r\cdot \opt + \ex[e^{|\w^*\cdot \x|}\ind\{|\w^*\cdot \x|> r\}]
    \end{align*}
    To bound the quantity $\ex[e^{|\w^*\cdot \x|}\ind\{|\w^*\cdot \x|> r\}]$, consider $F(s) := \Pr[e^{|\w^*\cdot \x|} \ind\{|\w^*\cdot \x| > r\} \ge s]$. We have that
    \[
        F(s) = \ind\{s=0\}\Pr[|\w^*\cdot \x|\le r] + \Pr[|\w^*\cdot \x| \ge \max\{\ln s, r\}]\,.
    \]
    Since $\ex[e^{|\w^*\cdot \x|}\ind\{|\w^*\cdot \x|> r\}] = \int_{s=0}^{\infty} F(s) \, ds = \int_{s=0}^\infty \Pr[|\w^*\cdot \x| \ge \max\{\ln s, r\}]\, ds$, we obtain
    \begin{align*}
        \ex[Q] & = \int_{s=0}^\infty \Pr\left[|\w^*\cdot \x| \ge \max\{\ln s, r\}\right]\, ds \\
        & \le \int_{s=0}^{e^{r}}\Pr\left[|\hat{\w}^*\cdot \x| \ge \frac{r}{B}\right]\, ds + \int_{s=e^{r}}^\infty \Pr\left[|\hat{\w}^*\cdot \x| \ge \frac{\ln s}{B}\right]\, ds \tag{since $\|\w^*\|_2\le \wnorm$}\\
        & \le e^r \cdot e^{-({r}/{\wnorm})^2} + \int_{s=e^r}^{\infty} e^{-(\ln s / \wnorm)^2}\, ds\tag{see Def. \ref{definition:concentration}}\\
        & = e^r \cdot e^{-({r}/{\wnorm})^2} + e^r\cdot \int_{u=0}^{\infty} e^{u-(\frac{u+r}{\wnorm})^2}\, du \tag{define $u = \ln s - r$}\\
        &\le e^r \cdot e^{-({r}/{\wnorm})^2} + e^r\cdot e^{-({r}/{\wnorm})^2} \cdot \wnorm\cdot e^{\frac{\wnorm^2}{2}}\cdot \int_{u=0}^{\infty} e^{-(u-\frac{\wnorm}{2})^2}\, du \\
        &\le Ce^{{\wnorm^2}} e^r e^{-(r/\wnorm)^2}
    \end{align*}
    Therefore, in total, we have that $\elltwoerror(p) \le 8e^r \opt_\g + 8Ce^{\wnorm^2} e^r e^{-(r/\wnorm)^2} + 2 \epsilon$ and we may obtain Theorem \ref{theorem:logistic-squared-expected} by picking $r=\wnorm(\ln\frac{1}{\opt})^{1/2}$.

\subsection{Proof of Theorem \ref{theorem:logistic-absolute-expected}}

We first prove Lemma \ref{lemma:kl-and-absolute-loss}. We have that for any $y\in\{0,1\}$ and $p\in(0,1)$
\[
    \ell_\g(y,\f'(p)) = y\ln\frac{1}{p} + (1-y) \ln\frac{1}{1-p} - \ln 2 = \crossentropy(y,p) - \ln 2\,,
\]
where $\crossentropy(y,p)$ is the cross entropy function. It is sufficient to show that for $y \in \zo$ and $p \in \izo$, 
    \begin{align}
    \label{eq:ce}
        |y -p | \leq \crossentropy(y,p ) \leq 2|y -p|\logf{1}{p(1 -p)}
    \end{align}

Observe that
\begin{align}
\label{eq:c0}
\crossentropy(0, p) = \logf{1}{1- p} = \sum_{i=1}^\infty \frac{p^i}{i}
\end{align}
where the series on the right converges for all $p < 1$. 
We can also write
\begin{align}
\label{eq:c1}
    \crossentropy(1, p) = \logf{1}{p} = \sum_{i=1}^\infty \frac{(1- p)^i}{i}
\end{align}
with the series converging for $p >0$.

    For the lower bound, we observe the following inequalities hold for all $p \in [0,1]$
    \begin{align*}
        \crossentropy(0, p) & \geq p = |0 - p |\\
        \crossentropy(1, p) & \geq 1 - p = | 1 - p|.
    \end{align*}
    For the upper bound, we first prove the claim for $y=0$, where it states that
    \begin{align} 
    \crossentropy(0, p) = \logf{1}{1- p} \leq 2p\logf{1}{p(1 -p)}.
    \end{align}
    
    When $p \leq 1/2$ we can use Eq. \eqref{eq:c0} to bound
    \begin{align}
    \crossentropy(0, p) \leq \sum_{i=1}^\infty \frac{p^i}{i} \leq \sum_{i=1}^\infty p^i = \frac{p}{1- p} \leq 2p.
    \end{align}
    The bounds holds by observing that since $p(1 -p) \geq 1/4$,
    \[ \logf{1}{p(1-p)} \geq \log(4) \geq 1\] 
    
    When $p \geq 1/2$, we note that
    \begin{align*}
    \logf{1}{p(1 -p)} \geq \logf{1}{1- p} = \crossentropy(0, p)
    \end{align*}
    and $2p \geq 1$. Hence the bound holds in this case too.

     In the case where $y = 1$, the bound states that
     \[ \crossentropy(1, p) = \logf{1}{p} \leq 2(1 -p)\logf{1}{p(1 -p)}.\]
     This is implied by our bound for $y =0$ by taking $q = 1 -p$. This concludes the proof of Lemma \ref{lemma:kl-and-absolute-loss} and we are ready to prove Theorem \ref{theorem:logistic-absolute-expected}. The following are true

     \begin{align}
    \elloneerror(p) &\le \L_{\g}(\f'\circ p) - \ln 2\label{equation:sigmoid-errorone-le-surr} \\
    \L_{\g}(\w^*) - \ln 2 &\le 2\cdot \ex\left[\ln\left(\frac{1}{\g'(\w^*\cdot \x) \cdot (1-\g'(\w^*\cdot \x))}\right) \cdot |y-\g'(\w^*\cdot \x)|\right]\label{equation:sigmoid-surr-le-distorted-errone} \\
    \L_{\g}(\f'\circ p) &\le \L_{\g}(\w^*) + \epsilon
\end{align}

Similarly to the proof of Theorem \ref{theorem:logistic-squared-expected}, we observe that
\[
    \ln\left(\frac{1}{\g'(\w^*\cdot \x) \cdot (1-\g'(\w^*\cdot \x))}\right) \le \ln 4 + |\w^*\cdot \x|
\]

It remains to bound the quantity $\ex[ {|\w^*\cdot \x|} \cdot |y - \g'_{\w}(\x)|]$. Set $Q = |\w^*\cdot \x|\cdot |y - \g'_{\w}(\x)| $ ($Q$ is a random variable). Then for any $r\ge 0$ we have
    \begin{align*}
        \ex[Q] &= \ex[Q\cdot \ind\{|\w^*\cdot \x|\le r\}] + \ex[Q\cdot \ind\{|\w^*\cdot \x|> r\}] \\
        &\le r \ex[|y-\g'_{\w^*}(\x)|\cdot \ind\{|\w^*\cdot \x|\le r\}] + \ex[{|\w^*\cdot \x|}\cdot \ind\{|\w^*\cdot \x|> r\}] \\
        & \le r\cdot \opt + \ex[{|\w^*\cdot \x|}\cdot\ind\{|\w^*\cdot \x|> r\}]
    \end{align*}
    To bound the quantity $\ex[{|\w^*\cdot \x|}\cdot \ind\{|\w^*\cdot \x|> r\}]$, consider $F(s) := \Pr[{|\w^*\cdot \x|}\cdot \ind\{|\w^*\cdot \x| > r\} \ge s]$. We have that
    \[
        F(s) = \ind\{s=0\}\Pr[|\w^*\cdot \x|\le r] + \Pr[|\w^*\cdot \x| \ge \max\{s, r\}]\,.
    \]
    Since $\ex[{|\w^*\cdot \x|}\cdot \ind\{|\w^*\cdot \x|> r\}] = \int_{s=0}^{\infty} F(s) \, ds = \int_{s=0}^\infty \Pr[|\w^*\cdot \x| \ge \max\{ s, r\}]\, ds$, we obtain
    \begin{align*}
        \ex[Q] & = \int_{s=0}^\infty \Pr\left[|\w^*\cdot \x| \ge \max\{ s, r\}\right]\, ds \\
        & \le \int_{s=0}^{{r}}\Pr\left[|\hat{\w}^*\cdot \x| \ge \frac{r}{B}\right]\, ds + \int_{s={r}}^\infty \Pr\left[|\hat{\w}^*\cdot \x| \ge \frac{s}{B}\right]\, ds \tag{since $\|\w^*\|_2\le \wnorm$}\\
        & \le r \cdot e^{-{r}/{\wnorm}} + \int_{s=r}^{\infty} e^{- s / \wnorm}\, ds\tag{see Def. \ref{definition:concentration}}\\
        & = r \cdot e^{-{r}/{\wnorm}} + \wnorm\cdot e^{-r/\wnorm}
    \end{align*}
    Therefore, in total, we have that $\elloneerror(p) \le 2\ln 4 \cdot \opt_\g + 2(r+B)e^{-r/B} + 2r\cdot \opt_\g + \epsilon$ and we may obtain Theorem \ref{theorem:logistic-absolute-expected} by picking $r=\wnorm\cdot \ln\frac{1}{\opt}$.

\end{document}